\documentclass[conference]{IEEEtran}

\usepackage{cite}

\usepackage{amsmath,amssymb,amsfonts}
\usepackage{algorithmic}
\usepackage{graphicx}
\usepackage{textcomp}
\usepackage{xcolor}
\usepackage{booktabs,array}
\usepackage{microtype}
\usepackage{dblfloatfix}
\usepackage{placeins}
\usepackage[hidelinks]{hyperref}
\usepackage{url}
\newtheorem{proposition}{Proposition}
\newcommand{\tightlist}{\setlength{\itemsep}{0pt}\setlength{\parskip}{0pt}}
\graphicspath{{figures/}}
\hypersetup{pdfauthor={Chen Yu, Sudarshan S Harithas, Juan Borbon, Dingqi Zhang, Jiuguang Wang}}

\title{Belief-Overlap Partial Inversion: Object-Centric Flow Replanning for Humanoid Loco-Manipulation}
\author{\IEEEauthorblockN{Chen Yu\IEEEauthorrefmark{1},
    Sudarshan S Harithas\IEEEauthorrefmark{2},
    Juan Borbon\IEEEauthorrefmark{2},
    Dingqi Zhang\IEEEauthorrefmark{2},
    Jiuguang Wang\IEEEauthorrefmark{2}}
  \IEEEauthorblockA{\IEEEauthorrefmark{1}Northwestern University, USA\\
    \IEEEauthorrefmark{2}Multiply Labs, USA}
  \thanks{This work was conducted while Chen Yu was an intern at Multiply Labs.}}
\IEEEoverridecommandlockouts
\IEEEaftertitletext{\vspace{-\baselineskip}\begin{minipage}{\textwidth}
\centering
\newlength{\teaserphotogap}
\setlength{\teaserphotogap}{2pt}
\newsavebox{\teaserimage}
\newcommand{\teaserpanelmark}[2]{\llap{\makebox[\wd\teaserimage][l]{\hspace*{#1}\raisebox{3pt}[0pt][0pt]{\begingroup\color{black}\normalfont\fontfamily{phv}\mdseries\fontsize{10}{10}\selectfont
        \makebox[0pt][r]{#2}\endgroup}}}}
\newcommand{\teasertimemark}[2]{\if\relax\detokenize{#2}\relax\else
    \llap{\makebox[\wd\teaserimage][l]{\hspace*{#1}\raisebox{3pt}[0pt][0pt]{\begingroup\color{black}\normalfont\fontfamily{phv}\mdseries\fontsize{7.5}{8}\selectfont
          #2\endgroup}}}\fi
}

\newcommand{\teaserphotodivider}[1]{\llap{\makebox[\wd\teaserimage][l]{\hspace*{#1}\smash{\textcolor{black!9}{\rule{0.5pt}{\ht\teaserimage}}}}}}
\newsavebox{\teasergoalimage}

\sbox{\teaserimage}{\includegraphics[page=1,width=\dimexpr(\textwidth-3\teaserphotogap)/4\relax]
        {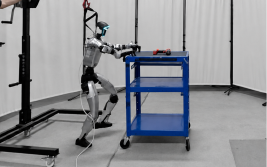}}\usebox{\teaserimage}\teaserpanelmark{\dimexpr\wd\teaserimage-4pt\relax}{A}\hspace{\teaserphotogap}\sbox{\teaserimage}{\includegraphics[width=\dimexpr(\textwidth-3\teaserphotogap)/4\relax]{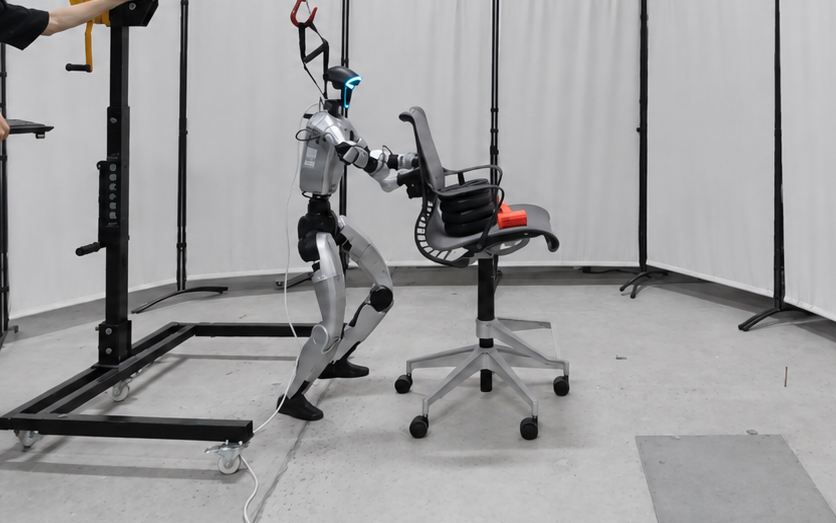}}\usebox{\teaserimage}\teaserpanelmark{\dimexpr\wd\teaserimage-4pt\relax}{B}\hspace{\teaserphotogap}\sbox{\teaserimage}{\includegraphics[page=2,width=\dimexpr(\textwidth-3\teaserphotogap)/4\relax]
        {figures/bopi_cart_photos_clean_cropped.pdf}}\usebox{\teaserimage}\teaserpanelmark{\dimexpr\wd\teaserimage-4pt\relax}{C}\hspace{\teaserphotogap}\sbox{\teaserimage}{\includegraphics[page=3,width=\dimexpr(\textwidth-3\teaserphotogap)/4\relax]
        {figures/bopi_cart_photos_clean_cropped.pdf}}\usebox{\teaserimage}\teaserpanelmark{\dimexpr\wd\teaserimage-4pt\relax}{D}\par
\vspace{\teaserphotogap}\nointerlineskip
\sbox{\teaserimage}{\includegraphics[width=\dimexpr(\textwidth-\teaserphotogap)/2\relax]{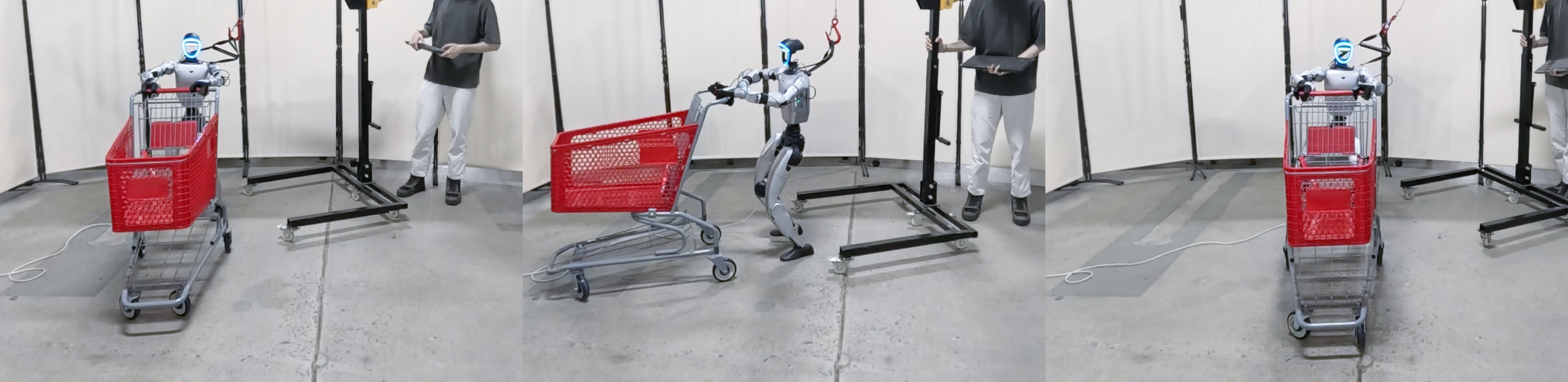}}\usebox{\teaserimage}\teaserphotodivider{\dimexpr\wd\teaserimage/3-0.25pt\relax}\teaserphotodivider{\dimexpr\wd\teaserimage*2/3-0.25pt\relax}\teaserpanelmark{\dimexpr\wd\teaserimage/3-4pt\relax}{E}\teaserpanelmark{\dimexpr\wd\teaserimage*2/3-4pt\relax}{F}\teaserpanelmark{\dimexpr\wd\teaserimage-4pt\relax}{G}\teasertimemark{4pt}{0 s}\teasertimemark{\dimexpr\wd\teaserimage/3+4pt\relax}{12 s}\teasertimemark{\dimexpr\wd\teaserimage*2/3+4pt\relax}{31 s}\sbox{\teasergoalimage}{\includegraphics[width=\dimexpr\wd\teaserimage/3\relax]{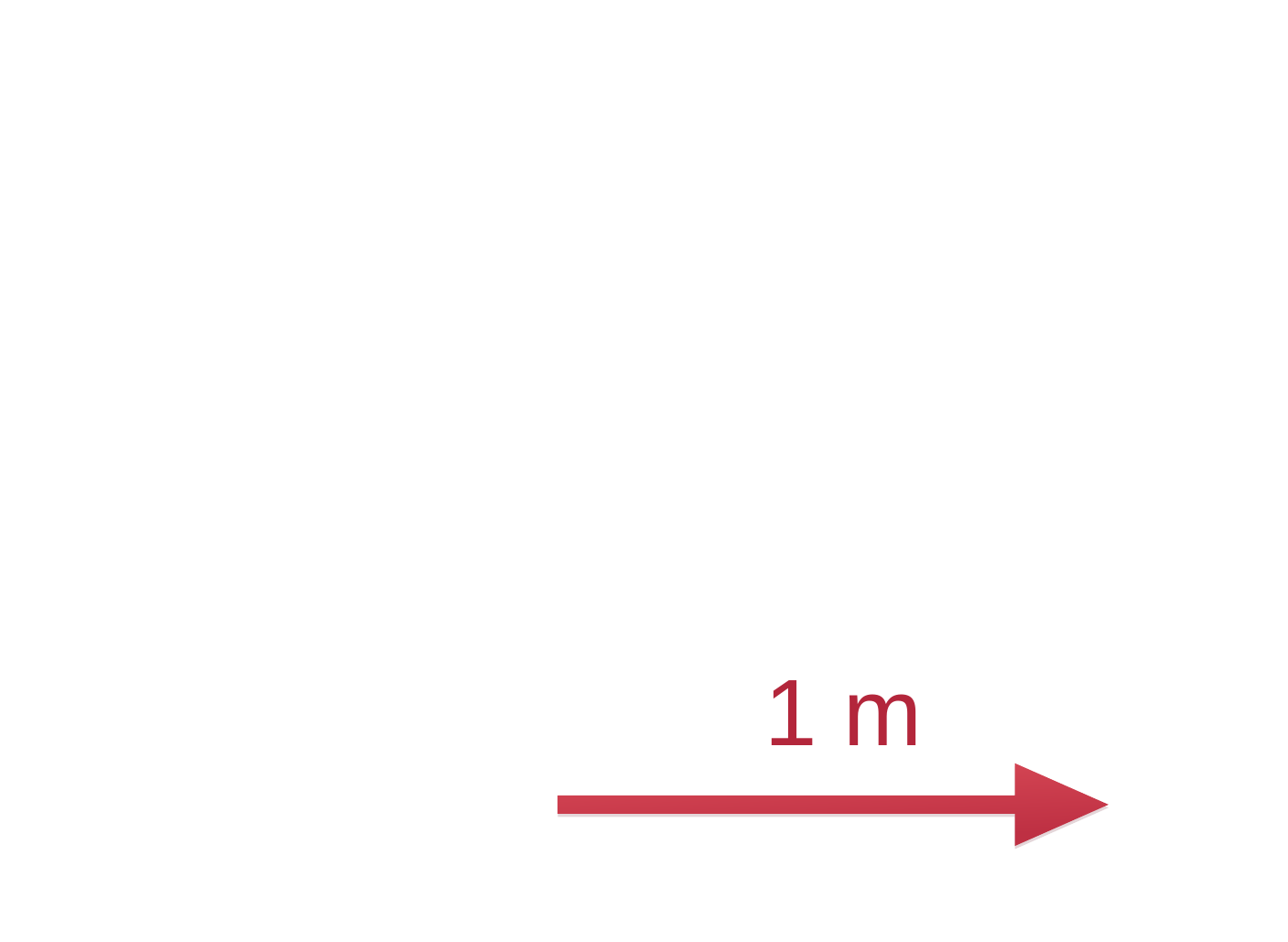}}\llap{\makebox[\wd\teaserimage][l]{\hspace*{0pt}\raisebox{\dimexpr\ht\teaserimage-\ht\teasergoalimage-0pt\relax}[0pt][0pt]{\usebox{\teasergoalimage}}}}\hspace{\teaserphotogap}\sbox{\teaserimage}{\includegraphics[width=\dimexpr(\textwidth-\teaserphotogap)/2\relax]{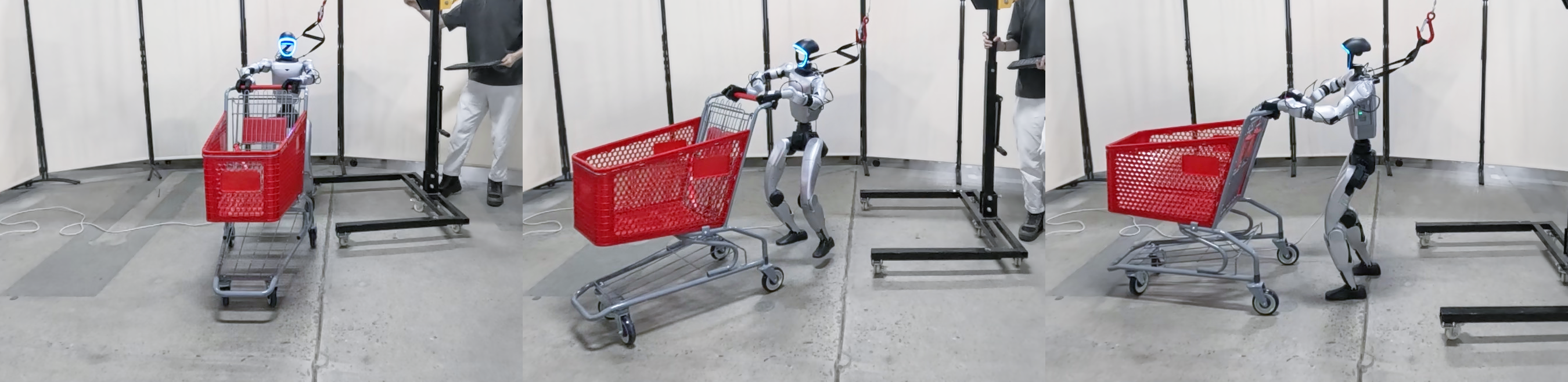}}\usebox{\teaserimage}\teaserphotodivider{\dimexpr\wd\teaserimage/3-0.25pt\relax}\teaserphotodivider{\dimexpr\wd\teaserimage*2/3-0.25pt\relax}\teaserpanelmark{\dimexpr\wd\teaserimage/3-4pt\relax}{H}\teaserpanelmark{\dimexpr\wd\teaserimage*2/3-4pt\relax}{I}\teaserpanelmark{\dimexpr\wd\teaserimage-4pt\relax}{J}\teasertimemark{4pt}{0 s}\teasertimemark{\dimexpr\wd\teaserimage/3+4pt\relax}{6 s}\teasertimemark{\dimexpr\wd\teaserimage*2/3+4pt\relax}{18 s}\sbox{\teasergoalimage}{\includegraphics[width=\dimexpr\wd\teaserimage/3\relax]{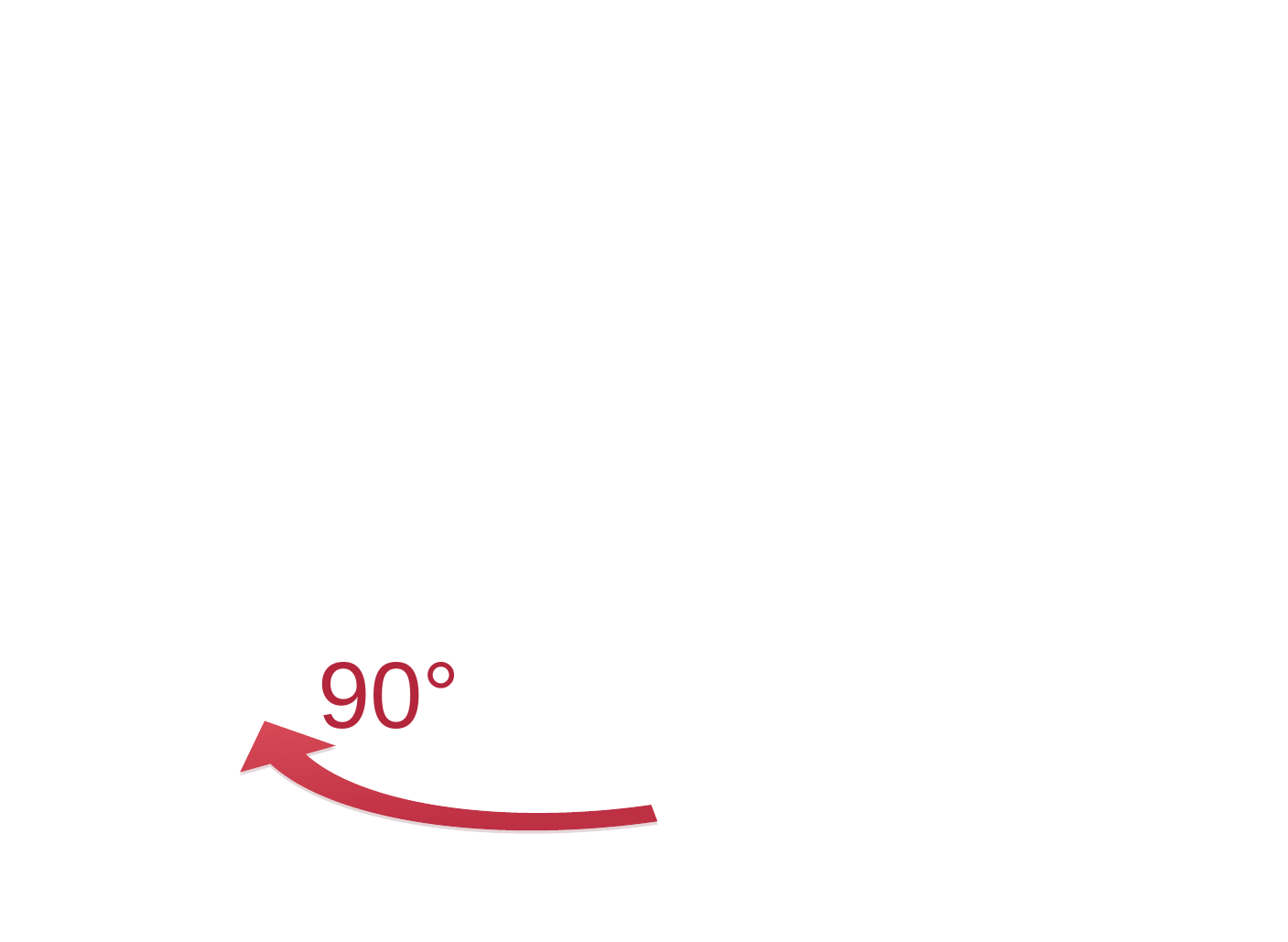}}\llap{\makebox[\wd\teaserimage][l]{\hspace*{0pt}\raisebox{\dimexpr\ht\teaserimage-\ht\teasergoalimage-0pt\relax}[0pt][0pt]{\usebox{\teasergoalimage}}}}\par
\nointerlineskip
\sbox{\teaserimage}{\includegraphics[width=\dimexpr(\textwidth-\teaserphotogap)/2\relax]
      {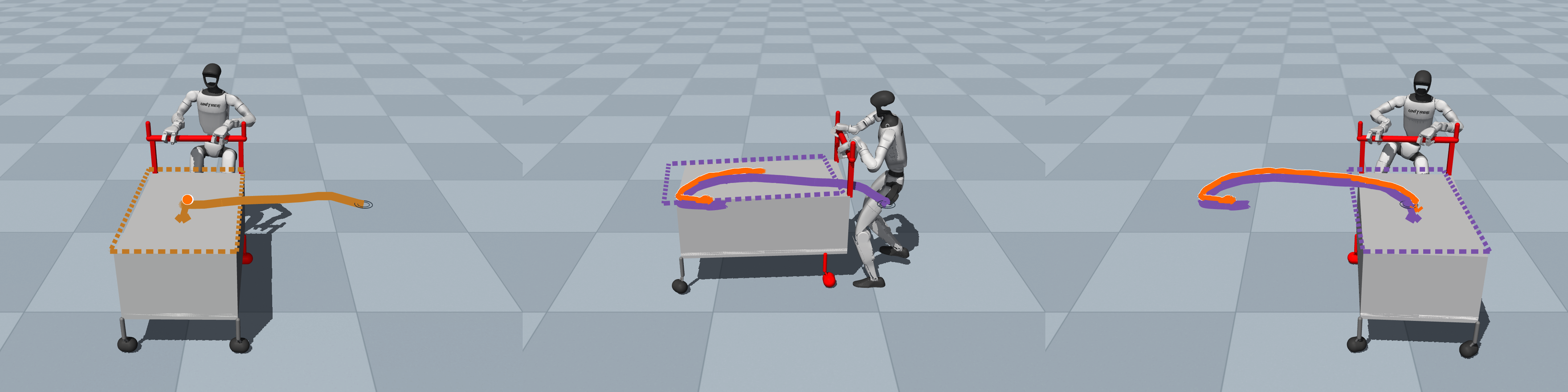}}\usebox{\teaserimage}\llap{\makebox[\wd\teaserimage][l]{\hspace*{\dimexpr3pt-1.5pt\relax}\raisebox{\dimexpr\ht\teaserimage-24pt-1.5pt\relax}[0pt][0pt]{\makebox[0pt][l]{\includegraphics[height=24pt]{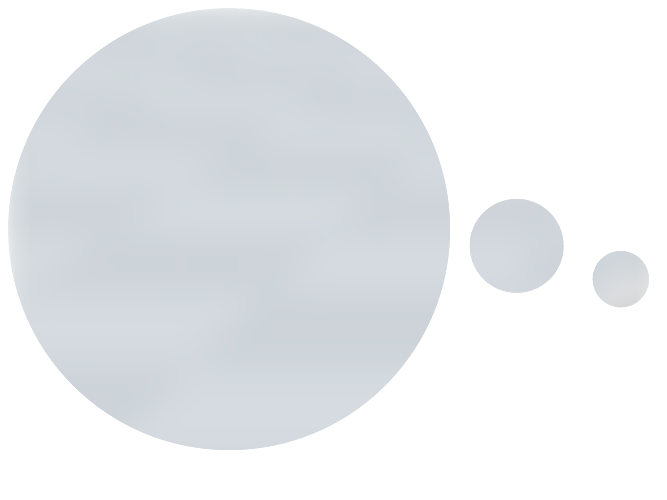}}\makebox[0pt][l]{\hspace*{3.5pt}\raisebox{5.5pt}{\includegraphics[height=15pt]{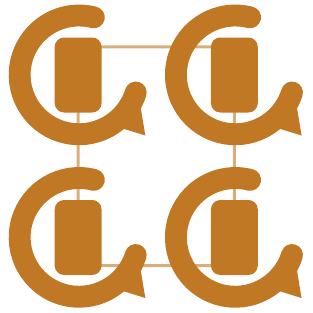}}}\hspace*{15pt}\raisebox{0.5pt}{\includegraphics[height=7pt]{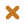}}}}}\llap{\makebox[\wd\teaserimage][l]{\hspace*{\dimexpr\dimexpr\wd\teaserimage/3+3pt\relax-1.5pt\relax}\raisebox{\dimexpr\ht\teaserimage-24pt-1.5pt\relax}[0pt][0pt]{\makebox[0pt][l]{\includegraphics[height=24pt]{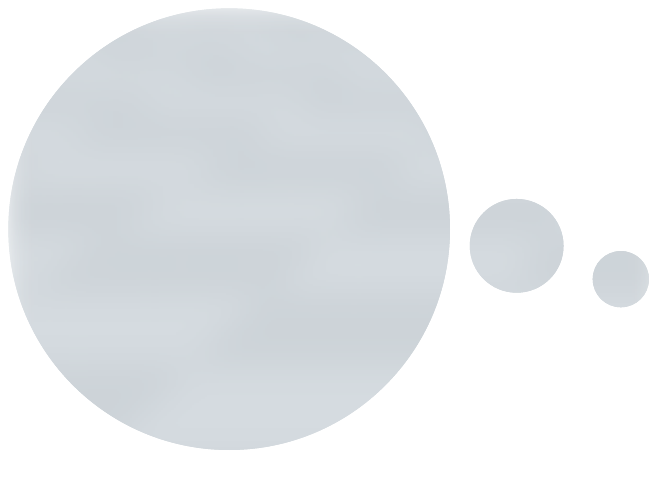}}\makebox[0pt][l]{\hspace*{3.5pt}\raisebox{5.5pt}{\includegraphics[height=15pt]{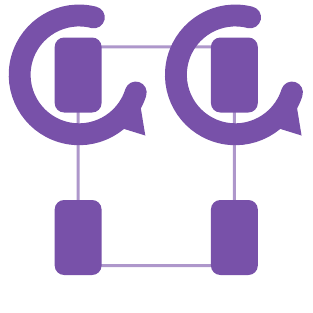}}}\hspace*{15pt}\raisebox{0.5pt}{\includegraphics[height=7pt]{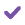}}}}}\llap{\makebox[\wd\teaserimage][l]{\hspace*{\dimexpr\dimexpr\wd\teaserimage*2/3+3pt\relax-1.5pt\relax}\raisebox{\dimexpr\ht\teaserimage-24pt-1.5pt\relax}[0pt][0pt]{\makebox[0pt][l]{\includegraphics[height=24pt]{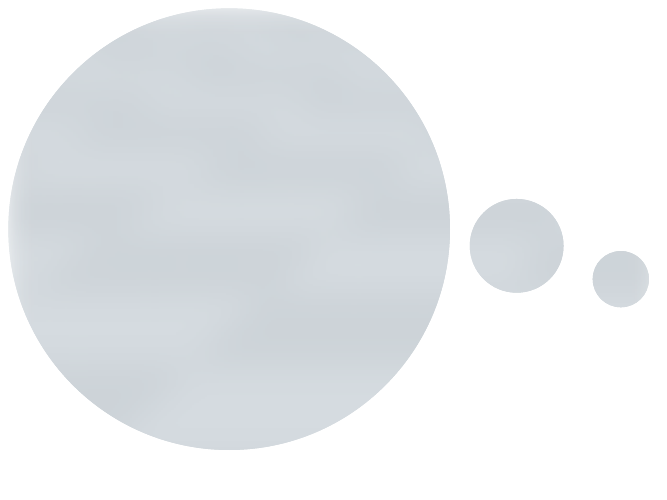}}\makebox[0pt][l]{\hspace*{3.5pt}\raisebox{5.5pt}{\includegraphics[height=15pt]{figures/cart_rear_locked_purple.pdf}}}\hspace*{15pt}\raisebox{0.5pt}{\includegraphics[height=7pt]{figures/cart_status_check.pdf}}}}}\teaserpanelmark{\dimexpr\wd\teaserimage/3-4pt\relax}{K}\teaserpanelmark{\dimexpr\wd\teaserimage*2/3-4pt\relax}{L}\teaserpanelmark{\dimexpr\wd\teaserimage-4pt\relax}{M}\teasertimemark{4pt}{0 s}\teasertimemark{\dimexpr\wd\teaserimage/3+4pt\relax}{12 s}\teasertimemark{\dimexpr\wd\teaserimage*2/3+4pt\relax}{31 s}\hspace{\teaserphotogap}\sbox{\teaserimage}{\includegraphics[width=\dimexpr(\textwidth-\teaserphotogap)/2\relax]
      {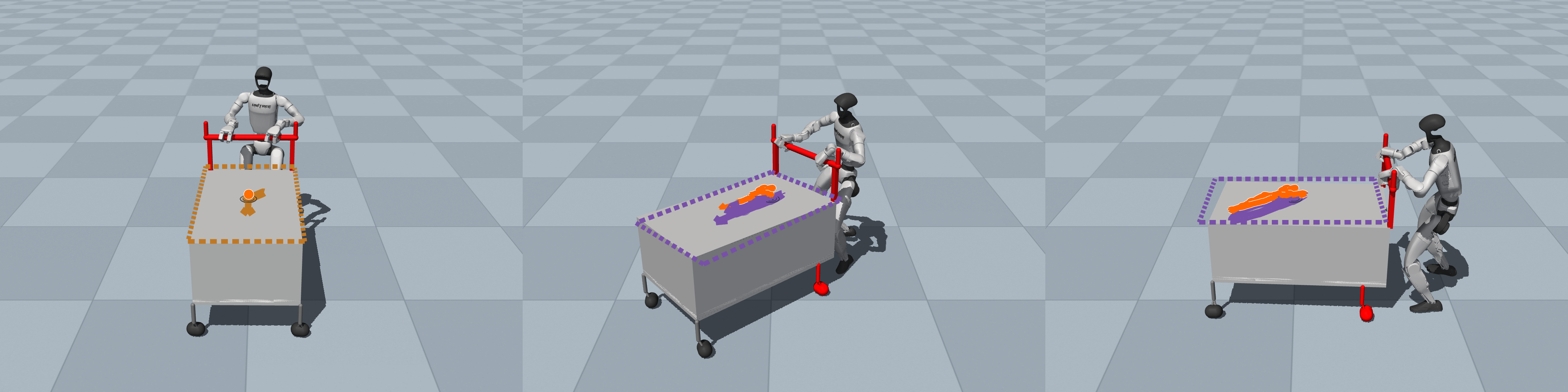}}\usebox{\teaserimage}\llap{\makebox[\wd\teaserimage][l]{\hspace*{\dimexpr3pt-1.5pt\relax}\raisebox{\dimexpr\ht\teaserimage-24pt-1.5pt\relax}[0pt][0pt]{\makebox[0pt][l]{\includegraphics[height=24pt]{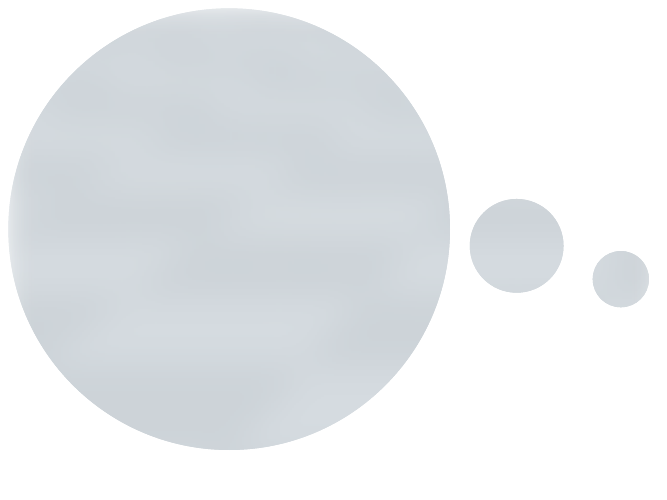}}\makebox[0pt][l]{\hspace*{3.5pt}\raisebox{5.5pt}{\includegraphics[height=15pt]{figures/cart_omni_gold.pdf}}}\hspace*{15pt}\raisebox{0.5pt}{\includegraphics[height=7pt]{figures/cart_status_cross.pdf}}}}}\llap{\makebox[\wd\teaserimage][l]{\hspace*{\dimexpr\dimexpr\wd\teaserimage/3+3pt\relax-1.5pt\relax}\raisebox{\dimexpr\ht\teaserimage-24pt-1.5pt\relax}[0pt][0pt]{\makebox[0pt][l]{\includegraphics[height=24pt]{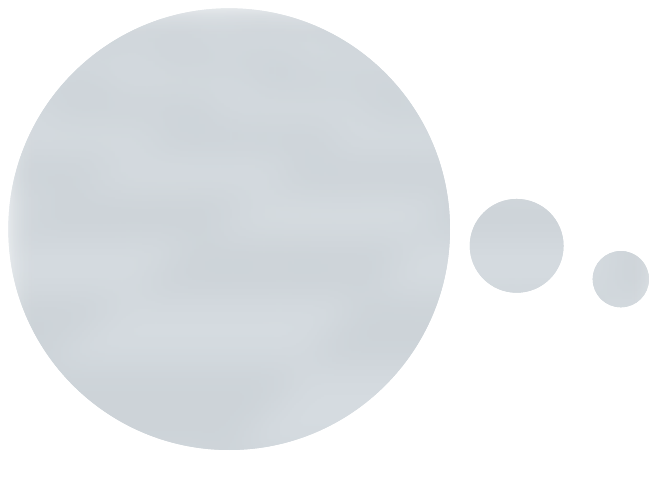}}\makebox[0pt][l]{\hspace*{3.5pt}\raisebox{5.5pt}{\includegraphics[height=15pt]{figures/cart_rear_locked_purple.pdf}}}\hspace*{15pt}\raisebox{0.5pt}{\includegraphics[height=7pt]{figures/cart_status_check.pdf}}}}}\llap{\makebox[\wd\teaserimage][l]{\hspace*{\dimexpr\dimexpr\wd\teaserimage*2/3+3pt\relax-1.5pt\relax}\raisebox{\dimexpr\ht\teaserimage-24pt-1.5pt\relax}[0pt][0pt]{\makebox[0pt][l]{\includegraphics[height=24pt]{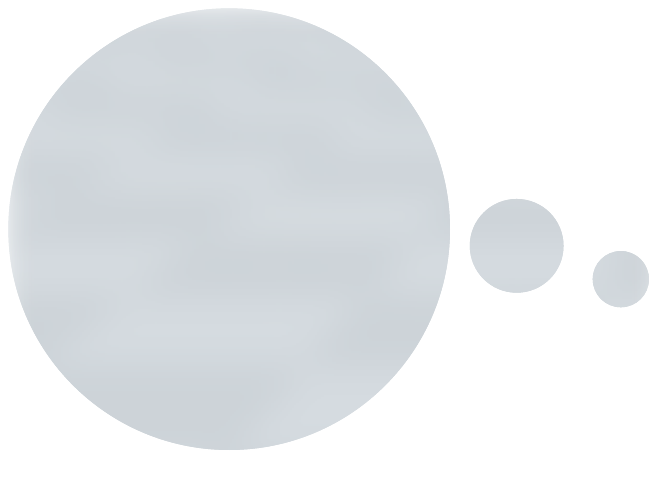}}\makebox[0pt][l]{\hspace*{3.5pt}\raisebox{5.5pt}{\includegraphics[height=15pt]{figures/cart_rear_locked_purple.pdf}}}\hspace*{15pt}\raisebox{0.5pt}{\includegraphics[height=7pt]{figures/cart_status_check.pdf}}}}}\teaserpanelmark{\dimexpr\wd\teaserimage/3-4pt\relax}{N}\teaserpanelmark{\dimexpr\wd\teaserimage*2/3-4pt\relax}{O}\teaserpanelmark{\dimexpr\wd\teaserimage-4pt\relax}{P}\teasertimemark{4pt}{0 s}\teasertimemark{\dimexpr\wd\teaserimage/3+4pt\relax}{6 s}\teasertimemark{\dimexpr\wd\teaserimage*2/3+4pt\relax}{18 s}\par
\makeatletter
\def\@captype{figure}
\makeatother
\caption{\textbf{From versatile object-centric control to adaptive planning.} Our whole-body controller pushes
wheeled objects with different masses, geometries, and holonomic or non-holonomic wheel constraints
along their planned paths (A--D). BOPI uses live estimates of the object's physical properties to revise lateral (E--G) and
turning (H--J) plans for a non-holonomic shopping cart. Corresponding simulations (K--P) illustrate this
adaptation: the planner proposes a \textcolor[HTML]{C07825}{wrong initial plan} based on an omni-cart
assumption. After brief interaction reveals the cart's non-holonomic constraint, BOPI revises it to a
\textcolor[HTML]{7851A9}{correct plan} while the cart continues
moving. Thought bubbles show each plan's assumed wheel constraints. Orange marks the cart's starting position
and executed motion. Videos and code at \href{https://bopi-loco.github.io}{\nolinkurl{bopi-loco.github.io}}.}
\vspace{4mm}
\label{fig:teaser}
\end{minipage}

  \vspace{-0.4em}}

\begin{document}
\maketitle
\vspace{-1.0ex}

\begin{abstract}
People often learn how an unfamiliar object moves through a small exploratory
push, using its response to guide their next motion. Inspired by this behavior, we study humanoids
pushing carts whose mass and caster stiffness are initially unknown. Our object-centric
framework closes online estimation, flow-based planning, and whole-body control into a loop,
creating a tension: replacing the plan whenever the belief shifts disrupts the maneuver under way,
while keeping it strands the robot on a path its evidence has ruled out. Belief-Overlap Partial
Inversion (BOPI) scales revision to agreement, inverting the flow under the old belief only as far
as the beliefs disagree before reconstructing under the new one. Across 18 simulated goals it lowers
pooled final position and yaw error by roughly 33\% and 47\% against fixed-depth replanning, and on
hardware it translates and turns a non-holonomic cart.
\end{abstract}
\section{Introduction}\label{introduction}
\begin{figure*}[!t]
\vspace*{-2mm}
\centering
\includegraphics[width=\textwidth]{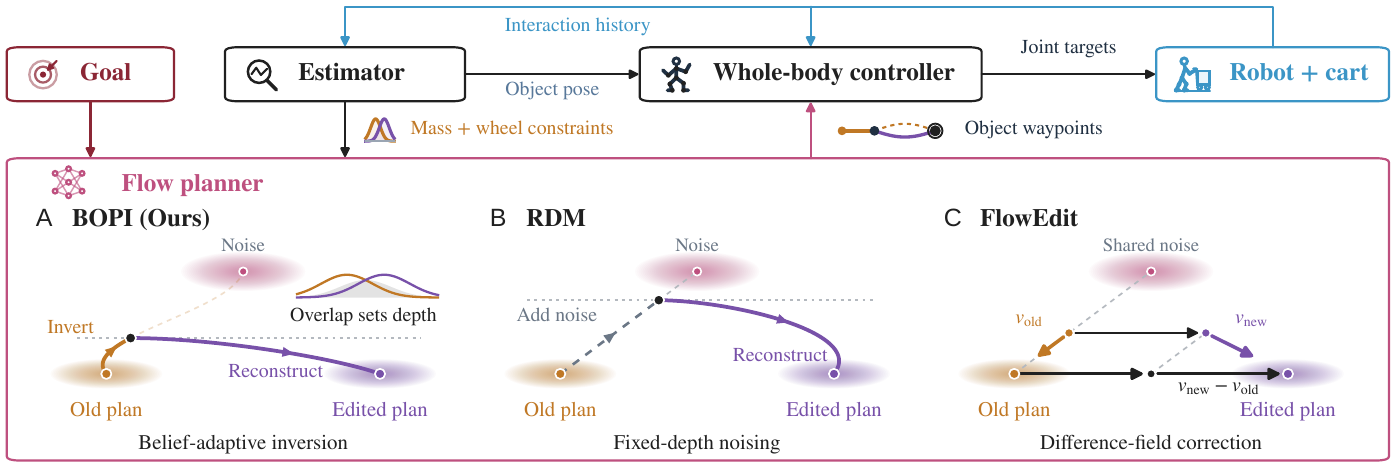}
\caption{\textbf{An object-centric estimation--planning--control framework.} Online estimation informs object planning and \textcolor{black}{whole-body control (trajectory waypoint tracking)}.
As estimates change, the planner can revise the remaining path using
(A) BOPI, which uses belief overlap to decide how far to run the flow backward under the old belief,
then runs it forward under the updated belief; (B) RDM~\cite{ref6}, which blends the current plan with a fixed
proportion of random noise, then runs the flow forward under the updated belief to reconstruct the path; or
(C) FlowEdit~\cite{ref7}, which uses shared noise to evaluate the flow under both beliefs, then edits the path
using the difference between the flow directions under the updated and old beliefs.
Orange/purple denote old/new plans and beliefs.
}
\vspace{-1mm}
\label{fig:bopi-framework}
\end{figure*}

People rarely commit to a way of moving an unfamiliar object before testing it; a short push reveals
whether it rolls freely, resists, or veers off line, and the plan is adjusted before the motion has
finished. We pursue this ability in humanoids pushing carts whose mass and caster stiffness are
unknown at the outset~\cite{ref8,ref1,li2026haic}. What is at stake is the path and not merely the
effort: an omnidirectional cart can be driven laterally to its goal, whereas a constrained rear axle
admits only a curved approach, so the geometry of the maneuver is settled by a property the robot
cannot observe until it is already pushing.

Our framework treats the cart, rather than the robot, as the thing being planned for. A recurrent
estimator maintains a belief over the cart's physical properties, a conditional flow planner
generates a timed object trajectory conditioned on that belief, and a whole-body controller tracks
that trajectory while preserving grasp and balance. Because motion is what produces evidence, the
three components close a loop in which every correction arrives mid-maneuver, and that timing is the
difficulty: regenerating the plan from each new belief, on top of noisy pose estimates, perturbs the
very trajectory the controller is following, while refusing to regenerate leaves the robot committed
to a path its own measurements no longer support. Building on plan-reuse
methods~\cite{ref6,black2025rtc,duan2025rtidp}, we ask:
\emph{how much should a new estimate change the robot's plan?}

\textbf{Belief-Overlap Partial Inversion (BOPI)} answers with the beliefs themselves. As the
replanning component of our framework, it partially reverses plan generation
under the old belief and reconstructs under the updated one, with the depth of reversal set by how
far the two beliefs overlap: when successive beliefs agree the plan is barely disturbed, and as they
diverge more of it is rewritten. The rule requires no per-situation tuning, and the change it
induces in the encoded plan vanishes as the overlap approaches one.

Experiments span a controlled nonlinear benchmark, where BOPI attains higher task success with
smaller cumulative plan changes, and 18 goals in G1--Dex3 simulation. Figure~\ref{fig:teaser} shows
object-centric control across cart types~(\mbox{A--D}), real-robot translation and
turning~(\mbox{E--J}), and simulation replay of recorded plans and updates~(\mbox{K--P}).

Our contributions are: \textbf{(i)} an object-centric estimation--planning--control framework that
adapts cart trajectories while tracking continues; \textbf{(ii)} BOPI, its belief-adaptive
replanning method, with an exact-integration bound under which the encoded
plan change vanishes as belief overlap approaches one; and \textbf{(iii)} simulation and hardware
evaluations showing lower goal error in simulation and lower final position and tracking error on
hardware.
\section{Related work}\label{related-work}

\subsection{Object-centric loco-manipulation}\label{loco-manipulation}
\label{object-centric-loco-manipulation}

Loco-manipulation couples locomotion with manipulation through whole-body
control~\cite{sleiman2021unified,zhang2025falcon,li2026haic,li2026vaic,byrd2026adaptmanip,ref9}. A
recurring move is to state the task on the object rather than on the robot. Mittal et
al.~\cite{mittal2022articulated} plan interactions with articulated objects in unknown scenes;
Omnigrasp~\cite{luo2024omnigrasp} and MaskedManipulator~\cite{tessler2025maskedmanipulator} drive
simulated humanoids from desired object motion; ResMimic~\cite{zhao2025resmimic} adds an
object-conditioned residual to a general motion tracker; and OmniContact~\cite{yu2026omnicontact}
chains contact-level skills for generalizable loco-manipulation. We share the object-level
specification with this line of work and differ in what happens afterwards: the object trajectory is
not fixed once issued but revised as the object's properties are inferred.

\subsection{Cart pushing under unknown dynamics}\label{cart-pushing}

Cart pushing is the classical instance of that coupling. Harada et al.~\cite{harada2004pushing}
derive walking motion for pushing manipulation, Scholz et al.~\cite{scholz2011cartpushing} treat the
cart as a movable obstacle during mobile-manipulator navigation, and Hawley and
Suleiman~\cite{hawley2017cartpushing,ref8} control a humanoid pushing a loaded cart and later
navigate with it under SLAM. Because the cart's response turns on parameters the robot cannot see,
several works estimate them while pushing: Sun et al.~\cite{sun2004cartidentification} identify cart
parameters online, Foster et al.~\cite{ref1} keep inertial estimates physically consistent for
humanoid loco-manipulation, Zhang et al.~\cite{zhang2025maneuverable} optimize the push pose and
reject disturbances without an accurate dynamic model, and Sriganesh et
al.~\cite{sriganesh2026actioninformed} couple estimation with planning for quadrupedal pedipulation.
Our question is orthogonal to how the estimate is formed: we ask what the planner should do with it
once it moves.

\subsection{Cross-embodiment control}\label{cross-embodiment-control}

Differences in physical properties or morphology change a system's dynamics and constraints, the
concern of cross-embodiment
control~\cite{openx2024rtx,yu2023multiembodiment,he2025omnih2o,ben2025homie,yang2025omniretarget}.
Rather than adapting control of the robot itself, we study how the object's path within the
robot--object system should change as its properties are inferred.

\subsection{Adaptive planning}\label{mechanics-aware-generative-planning}
\label{adaptive-trajectory-repair-and-execution}

Once properties are inferred~\cite{yu2017uposi,kumar2021rma,ref2}, the plan has to follow. Diffusion
and flow models~\cite{janner2022diffuser,chi2023diffusionpolicy,lipman2023flow,liu2023rectified,ref4}
generate trajectories that can be regenerated under a new condition, and several methods reuse the
previous solution instead of discarding it. SDEdit~\cite{ref5} noises an input and denoises it under
the new condition; Replanning with Diffusion Models (RDM)~\cite{ref6} applies this to replanning by
adding a fixed proportion of noise to the current plan before reconstruction.
FlowEdit~\cite{ref7} avoids inversion altogether, editing a trajectory by integrating the difference
between the flow fields the old and new conditions induce, while DDIM~\cite{song2021ddim} supplies
the deterministic inversion such methods rely on. A parallel line decides \emph{when} to re-solve
rather than how much of the old solution to keep, through real-time
chunking~\cite{black2025rtc,duan2025rtidp,feng2026dvac} and noise-space
steering~\cite{lu2026unisteer}. What none of these supplies is a principled setting for the amount
of revision: RDM's noise level is a fixed hyperparameter, and FlowEdit's edit follows from the field
difference alone. BOPI ties that amount to the overlap between the beliefs that produced the two
plans, so the plan changes as much as the evidence has.
\section{Method}\label{method}
Our framework separates object planning from whole-body control (Fig.~\ref{fig:bopi-framework}).
The planner specifies where the cart should move and when it should reach each pose. The whole-body \textcolor{black}{controller} follows this trajectory while maintaining balance and hand contact. As interaction
reveals more about the cart, BOPI updates the remaining path to better reach the goal.

\subsection{Object trajectories and the flow planner}\label{problem-formulation-contact-committed-loco-manipulation}

We represent a timed object trajectory as \(\tau=\{(p_h,t_h)\}_{h=0}^{H-1}\). Each of its \(H\)
waypoints records planar position and yaw in \(p_h\), together
with the elapsed physical time \(t_h\). The user specifies a fixed final position and yaw
\(g^{\mathrm W}\) in a world frame maintained by localization. At each replanning step, we express
this goal as \(g\): the remaining displacement and heading change needed to reach it from the
estimated current object pose. Our experiments use \(SE(2)\) goals, but the formulation permits
other goal representations with an appropriately conditioned flow.

An estimator uses interaction history to infer a belief \(q(\theta)\) over a vector
\(\theta\) of object physical properties, such as mass and caster stiffness.
The belief mean \(c(q)=\mu_q\), together with the goal \(g\), conditions the flow model.

A fixed encoder \(A\) converts the trajectory into normalized coordinates \(x=A(\tau)\).
A conditional rectified flow~\cite{lipman2023flow,liu2023rectified} uses flow time
\(\lambda=0\) for noise and \(\lambda=1\) for a plan. The intermediate encoding
\(x_\lambda\) evolves according to the flow field \(v_\varphi\), with network parameters
\(\varphi\), as \(dx_\lambda/d\lambda=v_\varphi(x_\lambda,\lambda\mid g,c(q))\).
We denote integration from flow time \(a\) to \(b\), starting at \(x\) with \(g\) and
\(c(q)\) fixed, by \(\Phi_q^{a\rightarrow b}(x;g)\). A fresh plan's encoding is
\(\Phi_q^{0\rightarrow1}(\epsilon;g)\), where \(\epsilon\sim\mathcal N(0,I)\)
and \(I\) is the identity matrix.

\subsection{Belief-Overlap Partial Inversion}\label{belief-overlap-partial-inversion}

BOPI repairs the remaining plan using the same trained flow. Let \(x_o\) encode this
plan, \(q_o\) denote its stored belief, and \(q_n\) the updated belief.

\paragraph{Overlap determines depth}
We measure agreement between the beliefs with their Bhattacharyya coefficient~\cite{bhattacharyya1946divergence,nielsen2022chernoff}. For continuous densities,
\begin{equation}\label{eq:bopi-overlap}
\begin{aligned}
\rho&=\operatorname{BC}(q_o,q_n)
=\int\sqrt{q_o(\theta)q_n(\theta)}\,d\theta,\\
d&=1-\rho,\qquad \lambda_\star=\rho.
\end{aligned}
\end{equation}
The overlap \(\rho\) lies between zero and one. We stop the backward integration at
\(\lambda_\star=\rho\), so it covers a distance \(d=1-\rho\) in flow time. Similar
beliefs give a shallow reversal; less overlap allows a deeper one. This depth is the squared
Hellinger distance and requires no extra threshold or repair-specific training.

\paragraph{Reverse the old plan, then reconstruct}
We reverse \(x_o\) from flow time \(1\) to \(\rho\) under \(q_o\), obtaining \(z_\rho\).
We then reconstruct from \(z_\rho\) under \(q_n\):
\begin{equation}\label{eq:bopi-repair}
\boxed{z_\rho=\Phi_{q_o}^{1\rightarrow\rho}(x_o;g),\qquad
x_n=\Phi_{q_n}^{\rho\rightarrow1}(z_\rho;g).}
\end{equation}
Both steps use the same \(g\).
The intermediate state is a trajectory encoding, not a physical pose for the robot to visit.
No fresh noise is added. At \(\rho=1\), both integrations have zero length and the plan
stays unchanged. At \(\rho=0\), they perform full inversion and reconstruction; values
between these extremes give partial repair.

\paragraph{Size of a repair}
The following bound describes how much the encoded plan can change under exact integration.
Write \(v_i(\cdot,\lambda)=v_\varphi(\cdot,\lambda\mid g,c(q_i))\) for \(i\in\{o,n\}\).

\begin{proposition}[Exact-flow repair]\label{prop:exact-flow-repair}
Suppose the solutions exist on \([\rho,1]\).
On a common region containing their paths, assume both fields are \(L\)-Lipschitz in encoded
state, bounded in norm by \(V\), and satisfy \(\|v_n-v_o\|_2\le\Delta_v\), uniformly on
\([\rho,1]\), with \(L,V,\Delta_v\ge0\). Then
\begin{equation}\label{eq:bopi-small-repair}
\|x_n-x_o\|_2\le\min\!\left\{2Vd,\;d\Delta_v\exp(Ld)\right\}.
\end{equation}
Here \(L\) bounds sensitivity to the input encoding, \(V\) bounds the field magnitude,
and \(\Delta_v\) bounds the difference caused by changing the belief.
\end{proposition}

\begin{IEEEproof}
Each integration travels at most \(Vd\) in encoded space, giving the first
bound, \(\|x_n-x_o\|_2\le2Vd\). For the second bound,
run the old flow forward from \(z_\rho\). By uniqueness,
this returns \(x_o\). The new flow starts from the same \(z_\rho\), so the distance
\(e(\lambda)\) between the two forward solutions satisfies \(e(\rho)=0\) and
\(e'\le Le+\Delta_v\) almost everywhere. Gr\"onwall's inequality gives
\[
\|x_n-x_o\|_2=e(1)\le\Delta_v\int_\rho^1 e^{L(1-s)}\,ds
\le d\Delta_v e^{Ld}.
\]
Taking the smaller of the two bounds proves Eq.~\eqref{eq:bopi-small-repair}.
\end{IEEEproof}

Under these uniform bounds, the change in the encoded plan approaches zero as belief
overlap approaches one.

\subsection{Applying updates while tracking continues}\label{online-deployment-contact-consistent-suffix-application}

While tracking continues, the estimator updates the belief \(q_n\).
Before repair, we resample the remaining path into \(H\) waypoints
\(p_h\) in the current estimated cart frame, matching the flow model's input representation.
BOPI revises the remaining timed path using the stored belief \(q_o\) and
updated belief \(q_n\), following Eq.~\eqref{eq:bopi-repair}, while keeping the final goal
\(g^{\mathrm W}\) unchanged. We preserve the first four resampled waypoints to reduce abrupt changes when updating path.

\subsection{Cart-pushing implementation}\label{planning-and-control-implementation}

\paragraph{Belief representation}
Gated recurrent unit (GRU) ensembles predict the mean and uncertainty of
rear-caster stiffness \(k_{\mathrm r}\) and cart mass \(m_{\mathrm{cart}}\) in standardized
log coordinates: \(\log k_{\mathrm r}\) and \(\log(1+m_{\mathrm{cart}})\).
Here \(k_{\mathrm r}\) denotes rear-caster swivel stiffness
(\(\mathrm{N\,m/rad}\)): small values allow the rear wheels to steer freely,
while large values approximate a locked rear axle.
We average the property estimates across models.
Let \(q_{i,j}=\mathcal N(\mu_{i,j},\sigma_{i,j}^2)\) denote the old
(\(i=o\)) or updated (\(i=n\)) belief for property \(j\).
Its mean is \(\mu_{i,j}\) and its standard deviation is \(\sigma_{i,j}>0\).
We compute \(\rho=\prod_j\operatorname{BC}(q_{o,j},q_{n,j})\) using the analytic
Gaussian overlap~\cite{nielsen2022chernoff}.

\paragraph{Generating training trajectories}
For each physical-property setting \(\theta\), an analytic teacher optimizes a sequence
of cart-frame controls \(\mathbf u\) (forward/lateral velocity and yaw rate) and duration \(T\):
\begin{equation}\label{eq:teacher-objective}
\min_{\mathbf u,T}\quad T+\sum_r w_r C_r(\mathbf u,T;\theta),
\end{equation}
subject to the goal pose, specified boundary velocities, and limits on motion and force/torque.
Each cost term \(C_r\), indexed by \(r\), sums a normalized squared quantity over time: predicted hand force, force rate,
acceleration, lateral velocity, yaw rate, or sideways velocity at an axle. The weights satisfy
\(w_r\ge0\). Mass changes the force and torque needed to accelerate the cart. Larger caster
stiffness increases the penalty on sideways axle motion and reduces its allowed magnitude.

\paragraph{Training the flow}
We convert each teacher trajectory into \(H=32\) timed waypoints
\((p_h,t_h)\) for flow training.
A temporal U-Net~\cite{janner2022diffuser,chi2023diffusionpolicy}, using GroupNorm~\cite{wu2018groupnorm} and FiLM conditioning~\cite{perez2018film}, predicts the flow field.

We use the standard flow-matching loss~\cite{lipman2023flow,liu2023rectified},
with training interpolation \(x_\lambda=(1-\lambda)\epsilon+\lambda x\).
The expectation is over training examples \((x,g,c(q))\) and independent
\(\epsilon\sim\mathcal N(0,I)\), \(\lambda\sim\mathcal U[0,1]\):
\begin{equation}\label{eq:flow-training}
\mathcal L_{\mathrm{FM}}
=\mathbb E\!\left[
\|v_\varphi(x_\lambda,\lambda\mid g,c(q))-(x-\epsilon)\|_2^2
\right].
\end{equation}
We also add a condition-smoothness loss to encourage similar flow predictions for nearby
physical-property inputs. With \(c=c(q)\), we perturb only \(c\), keeping
\(x_\lambda\), \(\lambda\), and \(g\) fixed:
\begin{equation}\label{eq:flow-smoothness}
\begin{aligned}
\mathcal L_{\mathrm{smooth}}=\mathbb E\Bigl[\bigl\|
&v_\varphi(x_\lambda,\lambda\mid g,c+\delta)\\
&-v_\varphi(x_\lambda,\lambda\mid g,c)\bigr\|_2^2\Bigr].
\end{aligned}
\end{equation}
Here \(\delta\sim\mathcal N(0,\sigma_c^2 I)\), with \(\sigma_c\simeq0.0041\) in
standardized physical-property coordinates. The total loss is
\(\mathcal L=\mathcal L_{\mathrm{FM}}+0.1\mathcal L_{\mathrm{smooth}}\).
All compared methods use this same trained flow.

\paragraph[Object-centric whole-body control]{\textcolor{black}{Object-centric whole-body control}}
The \textcolor{black}{controller} maps incoming timed waypoints \((p_h,t_h)\)
and robot observations to \textcolor{black}{targets for all joints except the fingers}.
Inputs include joint positions and velocities, base angular velocity, gravity direction,
previous actions, gait phase, hand targets, hand and handle geometry, and \textcolor{black}{an eight-step observation
history}.

We train the \textcolor{black}{controller} with PPO in mjlab~\cite{schulman2017ppo,zakka2026mjlab}.
Its main task rewards encourage accurate cart position and heading:
\begin{equation}\label{eq:tracker-reward}
r_{\mathrm{track}}
=15\exp\!\left(-\frac{\|e_p\|_2^2}{0.15^2}\right)
+20\exp\!\left(-\frac{e_\psi^2}{0.30^2}\right)\mathpunct{.}
\end{equation}
Here \(e_p\) and \(e_\psi\) measure planar position (m) and wrapped heading (rad)
differences from the current target interpolated from the timed waypoints \((p_h,t_h)\).
Supporting terms encourage stable grasping, balance, wheel contact, and terminal settling,
while penalizing excessive speed and sideways wheel slip. We combine these with
FALCON's locomotion rewards~\cite{zhang2025falcon} and will release the complete reward
specification and training configuration with our code.

\section{Experiments}\label{experiments}
\subsection{Evaluation protocol and baselines}\label{evaluation-protocol-and-baselines}

We compare BOPI with four baselines:
\begin{itemize}
\tightlist
\item \textbf{RDM} adds noise to the old plan at a fixed depth, then reconstructs
  it under the updated belief~\cite{ref6} (Fig.~\ref{fig:bopi-framework}B).
\item \textbf{FlowEdit} performs inversion-free repair using the difference between new-
  and old-conditioned flow fields evaluated at paired noisy states~\cite{ref7} (Fig.~\ref{fig:bopi-framework}C).
\item \textbf{Fresh FM} generates a complete plan from fresh Gaussian noise under the updated belief.
\item \textbf{Retain} keeps the initial plan unchanged.
\end{itemize}

\begin{figure}[!t]
\centering
\includegraphics[width=\columnwidth]{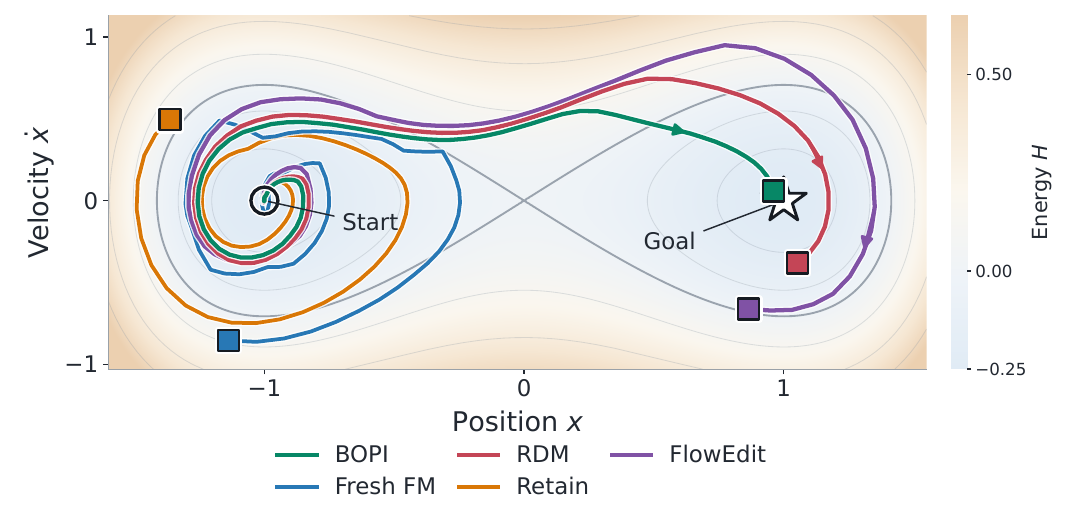}
\caption{\textbf{Planning in a nonlinear toy problem.} At 10~s (squares), BOPI is
near the goal at low speed; RDM and FlowEdit overshoot with higher residual speeds.
Background color shows mechanical energy \(H\).}
\label{fig:duffing-phase-portrait}
\end{figure}

\begin{table}[!t]
\caption{Toy problem results.}
\label{tab:results-1}
\centering\footnotesize
\setlength{\tabcolsep}{3pt}
\begin{tabular*}{\columnwidth}{@{\extracolsep{\fill}}lrrrrr@{}}
\toprule
Method & \shortstack{Success\\(\%)} & \shortstack{Position\\error} & \shortstack{Final speed\\error} & \shortstack{Cumulative\\mutation} & \shortstack{Strategy\\switches} \\
\midrule
\textbf{BOPI (Ours)} & \(\mathbf{54.9\pm14.6}\) & \textbf{0.189} & \textbf{0.116} & \textbf{0.704} & \textbf{2.144} \\
RDM & \(40.0\pm6.1\) & 0.283 & 0.211 & 2.212 & 8.617 \\
FlowEdit & \(12.6\pm3.7\) & 1.043 & 0.367 & 4.026 & 9.260 \\
Fresh FM & \(4.0\pm0.9\) & 1.567 & 0.422 & 23.631 & 21.045 \\
Retain & \(2.4\pm0.3\) & 2.090 & 0.491 & N/A & N/A \\
\bottomrule
\end{tabular*}
\end{table}

\begin{figure}[!b]
\centering
\includegraphics[width=\columnwidth]{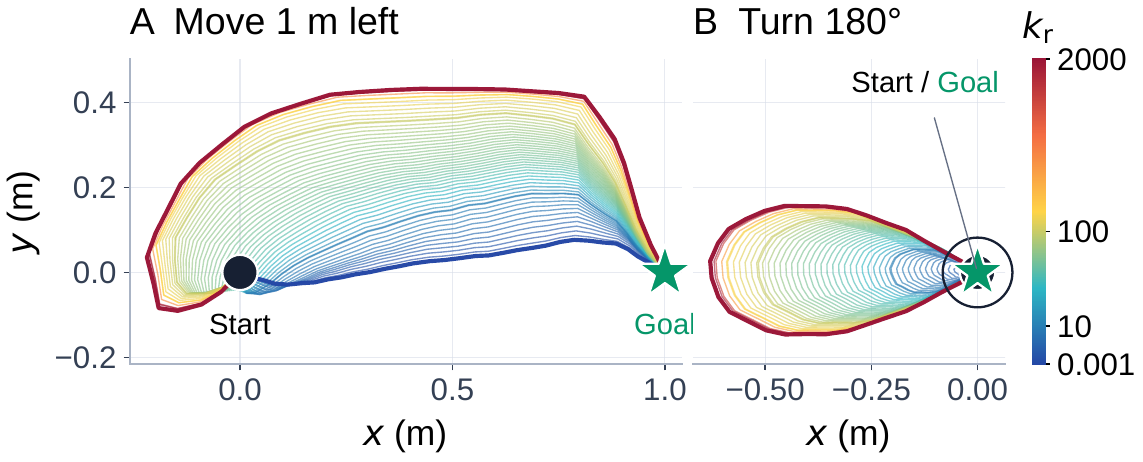}
\caption{\textbf{Flow-generated plans adapt to wheel constraints.}
With the same initial flow noise, the flow generates different paths as rear-caster stiffness changes:
move one meter left (A), or turn \(180^\circ\) and return to the starting position (B).
Color indicates stiffness; bold curves mark its extremes.}
\label{fig:flow-mechanics-family}
\end{figure}

\subsection{Controlled toy problem}\label{controlled-toy-problem}

\paragraph{Setup}\label{duffing-cross-well-benchmark-design}
We study plan repair on a controlled Duffing oscillator~\cite{suhardjo1992duffing},
\(\ddot x=x-x^3-\beta\dot x+0.43u\), with position \(x\), velocity \(\dot x\),
control \(u\in[-1,1]\), and damping \(\beta\sim\mathcal U[0.04,0.42]\).
Flows learn optimized controls for moving from \((x,\dot x)=(-1,0)\) to \((1,0)\)
over 10 seconds. Each trial starts with a teacher plan for \(\beta=0.04\).
Scripted equal-variance Gaussian beliefs first correct the damping, then fluctuate with
correlated noise (belief/noise standard deviation \(0.03\), correlation \(0.95\)).
At 4 Hz, methods revise the full plan but execute only future controls.
RDM uses the fixed midpoint flow restart time \(\lambda=0.50\).
Terminal success requires \(|x-1|<0.12\) and \(|\dot x|<0.15\).
We evaluate each method on 256 trials per flow across four independently trained flows.

\paragraph{Results}\label{cross-well-results}
An incorrect damping estimate makes the initial plan unreliable: retaining it succeeds in only
2.4\% of trials, compared with 97.2\% when planning with true damping. BOPI improves success
over fixed-depth RDM by \(14.8\pm10.4\) percentage points, while reducing cumulative plan
changes by 68.2\% and final speed error by 45.1\% (Table~\ref{tab:results-1}).
Figure~\ref{fig:duffing-phase-portrait} shows a representative successful BOPI rollout.

\subsection{Belief-conditioned flow generation}\label{flow-model-results}

In the toy problem, beliefs about damping guide control generation; for cart pushing, beliefs
about mass and wheel constraints guide the generation of timed cart trajectories.
Figure~\ref{fig:flow-mechanics-family} \textcolor{black}{shows how wheel constraints change the paths:
nearly free rear casters allow} direct lateral motion \textcolor{black}{and} near-in-place rotation,
\textcolor{black}{whereas a constrained rear axle produces curved lateral paths and wider turning maneuvers}.
\textcolor{black}{Beyond path shape, the plans also adapt their timing: for example, at}
\(m_{\mathrm{cart}}=18.3\,\mathrm{kg}\), \textcolor{black}{increasing rear-caster stiffness extends}
the lateral plan's duration from \(5.11\) to \(9.42\,\mathrm{s}\).
\textcolor{black}{Separately, the increase in mass belief from \(3.3\) to \(18.3\,\mathrm{kg}\)
at \(k_{\rm r}=2000\) increases in the lateral plan length by \(8.5\,\mathrm{cm}\).}

\begin{figure}[!t]
\centering
\includegraphics[width=\columnwidth]{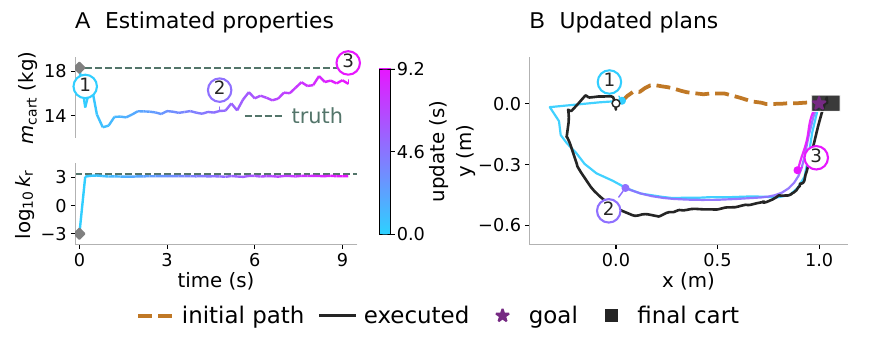}
\caption{\textbf{From property estimates to updated plans.}
As estimates of cart dynamic properties evolve (A), BOPI revises the initial near-straight path into a curved maneuver (B).
Matching colors and numbers link estimates to plan updates; black shows the executed motion.}
\label{fig:estimator-to-plan}
\end{figure}

\begin{figure}[t]
\centering
\includegraphics[width=\columnwidth]{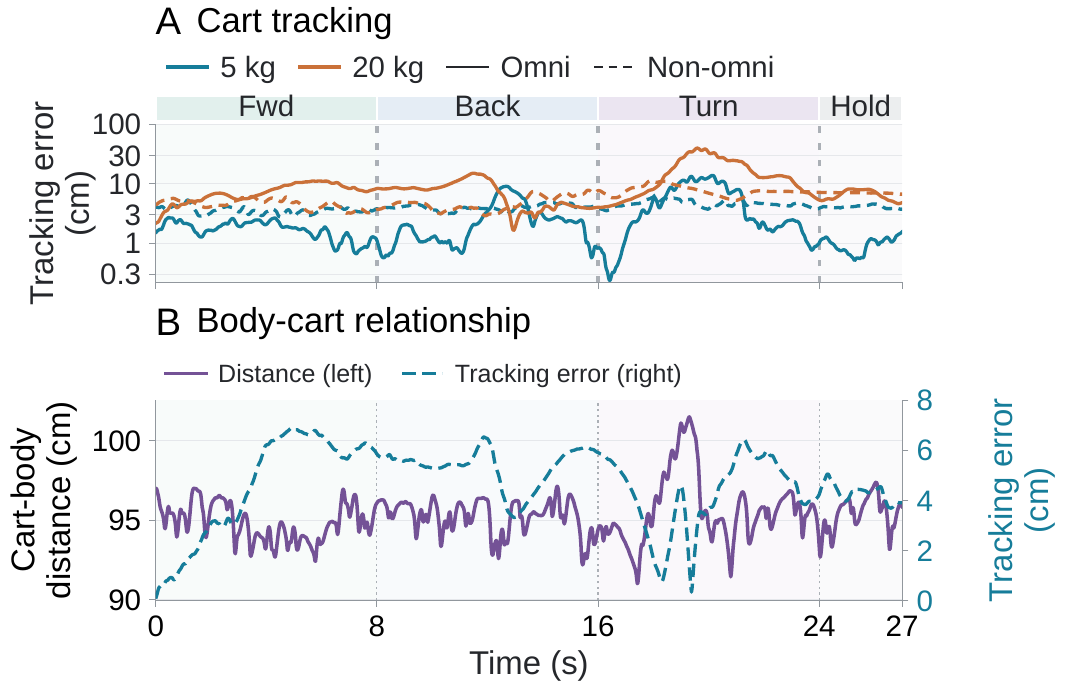}
\caption{\textbf{Accurate cart tracking with a changing body--cart distance.}
We design a path covering common maneuvers: \(0.8\,\mathrm{m}\) forward travel,
\(0.6\,\mathrm{m}\) reverse travel without changing heading, and a \(45^\circ\) rolling turn.
Panel A compares tracking across masses and wheel constraints. In an example execution with a
\(20\,\mathrm{kg}\) cart (B), body--cart distance changes during turning while tracking error
remains low, illustrating that the robot can follow the object's path without maintaining a fixed body--cart offset.}
\label{fig:tracker-maneuvers}
\end{figure}

\begin{table*}[!t]
\caption{Simulation results: end-to-end planning. Entries are mean final position error (m) / yaw error (rad).}
\label{tab:results-2}
\centering\small
\setlength{\tabcolsep}{7pt}
\begin{tabular}{@{}lrrrrr@{}}
\toprule
Task category & \textbf{BOPI (Ours)} & RDM & FlowEdit & Fresh FM & Retain \\
\midrule
Translation & \textbf{0.349 / 0.163} & 0.625 / 0.561 & 0.637 / 0.590 & 0.870 / 1.039 & 0.927 / 0.330 \\
Rotation & \textbf{0.109 / 0.150} & 0.134 / 0.195 & 0.251 / 0.693 & 0.754 / 1.209 & 0.286 / 0.306 \\
Translation + rotation & \textbf{0.209 / 0.205} & 0.296 / 0.336 & 0.240 / 0.228 & 0.835 / 1.165 & 0.495 / 0.329 \\
Translation + rotation (wrong mass) & \textbf{0.105 / 0.039} & 0.163 / 0.139 & 0.277 / 0.129 & 1.353 / 1.114 & 0.352 / 0.115 \\
\textbf{Pooled} & \textbf{0.179 / 0.152} & 0.267 / 0.286 & 0.295 / 0.349 & 0.936 / 1.149 & 0.465 / 0.277 \\
\bottomrule
\end{tabular}
\end{table*}

\begin{figure*}[!tp]
\centering
\includegraphics[width=1.0\textwidth]{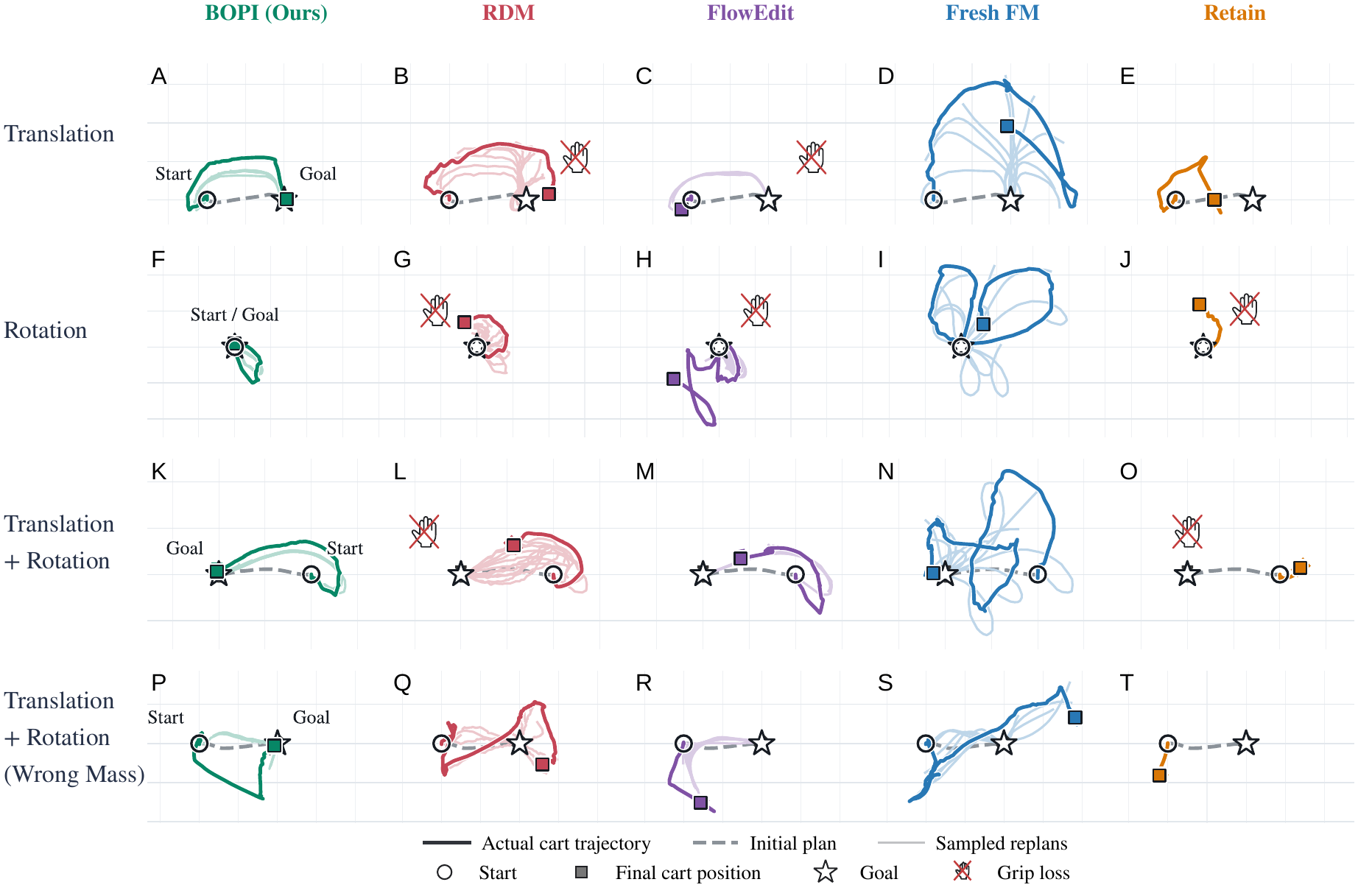}
\caption{\textbf{BOPI maintains goal-directed motion during replanning in simulation.}
\textcolor{black}{Planning examples span translation (A--E), rotation (F--J), combined translation--rotation (K--O),
and combined motion with an incorrect initial mass estimate (P--T).
BOPI makes smaller changes between successive plans, producing consistent replanned trajectories
as the robot progresses toward the goal. This consistency helps explain the lower final position
and yaw errors reported in Table~\ref{tab:results-2}.}
Grid spacing is \(0.5\,\mathrm{m}\) in both directions.}
\label{fig:replanning-executions}
\vspace{-1.6mm}
\end{figure*}

\subsection{Humanoid cart-pushing simulation}\label{humanoid-cart-pushing-simulation}

\subsubsection{Simulation setup}\label{simulation-setup}

We evaluate all methods in the mjlab~\cite{zakka2026mjlab} G1--Dex3 environment, using a
whole-body \textcolor{black}{controller} while the robot grasps the cart with both hands. At each reset, we randomize the robot
posture and handle height over a range of \(0.78\)--\(1.00\,\mathrm{m}\), then use inverse
kinematics to place both hands on the handle. During \textcolor{black}{controller} training, we also sample
episodes without a cart so that the deployed policy remains stable when contact is
not established.

\subsubsection{Belief estimator}\label{belief-estimator}

The belief estimator predicts cart mass and rear-wheel constraint from interaction history.
After four seconds of interaction, the estimator distinguishes free from constrained rear wheels
with \(91.5\%\)--\(98.4\%\) accuracy and estimates cart mass with approximately
\(1\,\mathrm{kg}\) mean absolute error on held-out trajectories.

Figure~\ref{fig:estimator-to-plan} illustrates how updated estimates guide BOPI's path repairs.
In the illustrated rollout, the cart reaches the goal with mean position and yaw errors of
\(0.036\,\mathrm{m}\) and \(0.007\,\mathrm{rad}\).

\subsubsection[Object-centric whole-body control]{\textcolor{black}{Object-centric whole-body control}}\label{whole-body-tracker}

Across masses and wheel constraints, our \textcolor{black}{controller} follows cart waypoints with a mean
RMS position error of \(6.53\,\mathrm{cm}\) over 8 trials (Fig.~\ref{fig:tracker-maneuvers}A).
Object-centric tracking allows the robot to adjust its position relative to the cart during turning while continuing
to track the cart's waypoints (Fig.~\ref{fig:tracker-maneuvers}B).

\subsubsection{End-to-end planning results}\label{simulation-results}

\textcolor{black}{We evaluate 18 goals (32 seeds each): one-meter left/right translations;
\(\pm60^\circ\) or \(\pm90^\circ\) rotations ending at the starting position;
one-meter left/right/forward/backward translations with \(\pm90^\circ\) turns;
and one-meter left/right translations with \(\pm45^\circ\) turns under an incorrect mass belief.}
\textcolor{black}{Directions are relative to the initial pushing direction.}
The first three categories incorrectly assume freely steering rear wheels\textcolor{black}{;} \textcolor{black}{wrong}-mass tasks use
an omnidirectional cart with an assumed mass of \(3.3\,\mathrm{kg}\) instead of \(18.3\,\mathrm{kg}\).
\textcolor{black}{We compute cart position and yaw errors from ground-truth poses against the fixed
goal \(g^{\mathrm W}\), averaging over each rollout's final two seconds and then across all rollouts.}
\textcolor{black}{Replanning is scheduled} at 5~Hz\textcolor{black}{;}
RDM's restart time \(\lambda=0.25\) was selected by a development sweep.

\textcolor{black}{Table~\ref{tab:results-2} shows that retaining the incorrect initial plan leaves substantial final position and yaw errors
(\(0.465\,\mathrm{m}\) / \(0.277\,\mathrm{rad}\)), while repeatedly replacing it with Fresh FM
yields even larger errors (\(0.936\,\mathrm{m}\) / \(1.149\,\mathrm{rad}\)).
BOPI instead scales revision to belief agreement, achieving the lowest errors in every category
(\(0.179\,\mathrm{m}\) / \(0.152\,\mathrm{rad}\) overall,
approximately \(33\%\) / \(47\%\) below RDM and \(39\%\) / \(56\%\) below FlowEdit).
Fig.~\ref{fig:replanning-executions} provides further evidence: BOPI makes smaller changes between successive plans,
producing consistent replanned trajectories as the robot progresses toward the goal.
Together, these results favor adapting the plan to updated estimates while limiting unnecessary
changes to the path the robot is following.}

\begin{table*}[!t]
\begin{minipage}[t]{\dimexpr(\textwidth-\columnsep)/2\relax}
\caption{Real-robot results: translation.}
\label{tab:real-robot-results}
\centering\small
\setlength{\tabcolsep}{1pt}
\renewcommand{\arraystretch}{1.15}
\begin{tabular*}{\linewidth}{@{\extracolsep{\fill}}lrrrr@{}}
\toprule
Method & \shortstack{Position\\error\\(cm)} & \shortstack{Yaw\\error\\($^\circ$)} & \shortstack{Position\\tracking RMS\\(cm)} & \shortstack{Yaw\\tracking RMS\\($^\circ$)} \\
\midrule
\textbf{BOPI (Ours)} & \shortstack[r]{\(\mathbf{6.48}\)\\\(\mathbf{\pm 3.48}\)} & \shortstack[r]{\(1.64\)\\\(\pm 1.52\)} & \shortstack[r]{\(\mathbf{7.20}\)\\\(\mathbf{\pm 1.41}\)} & \shortstack[r]{\(\mathbf{3.89}\)\\\(\mathbf{\pm 0.59}\)} \\
RDM & \shortstack[r]{\(7.35\)\\\(\pm 2.98\)} & \shortstack[r]{\(\mathbf{1.17}\)\\\(\mathbf{\pm 1.23}\)} & \shortstack[r]{\(8.98\)\\\(\pm 1.93\)} & \shortstack[r]{\(3.98\)\\\(\pm 0.60\)} \\
FlowEdit & \shortstack[r]{\(30.64\)\\\(\pm 40.41\)} & \shortstack[r]{\(17.33\)\\\(\pm 31.83\)} & \shortstack[r]{\(16.44\)\\\(\pm 13.47\)} & \shortstack[r]{\(6.73\)\\\(\pm 5.06\)} \\
Fresh FM & \shortstack[r]{\(43.32\)\\\(\pm 29.40\)} & \shortstack[r]{\(21.08\)\\\(\pm 16.04\)} & \shortstack[r]{\(13.01\)\\\(\pm 3.70\)} & \shortstack[r]{\(6.60\)\\\(\pm 1.36\)} \\
Retain & \shortstack[r]{\(85.48\)\\\(\pm 12.11\)} & \shortstack[r]{\(5.37\)\\\(\pm 4.63\)} & \shortstack[r]{\(40.58\)\\\(\pm 5.25\)} & \shortstack[r]{\(7.22\)\\\(\pm 0.85\)} \\
\bottomrule
\end{tabular*}
\end{minipage}\hfill
\begin{minipage}[t]{\dimexpr(\textwidth-\columnsep)/2\relax}
\caption{Real-robot results: turning.}
\label{tab:real-robot-rotation}
\centering\small
\setlength{\tabcolsep}{1pt}
\renewcommand{\arraystretch}{1.15}
\begin{tabular*}{\linewidth}{@{\extracolsep{\fill}}lrrrr@{}}
\toprule
Method & \shortstack{Position\\error\\(cm)} & \shortstack{Yaw\\error\\($^\circ$)} & \shortstack{Position\\tracking RMS\\(cm)} & \shortstack{Yaw\\tracking RMS\\($^\circ$)} \\
\midrule
\textbf{BOPI (Ours)} & \shortstack[r]{\(\mathbf{5.07}\)\\\(\mathbf{\pm 2.47}\)} & \shortstack[r]{\(4.51\)\\\(\pm 1.63\)} & \shortstack[r]{\(\mathbf{10.20}\)\\\(\mathbf{\pm 2.24}\)} & \shortstack[r]{\(\mathbf{5.05}\)\\\(\mathbf{\pm 0.79}\)} \\
RDM & \shortstack[r]{\(9.83\)\\\(\pm 3.89\)} & \shortstack[r]{\(\mathbf{3.11}\)\\\(\mathbf{\pm 1.68}\)} & \shortstack[r]{\(16.05\)\\\(\pm 4.51\)} & \shortstack[r]{\(6.43\)\\\(\pm 0.79\)} \\
FlowEdit & \shortstack[r]{\(8.29\)\\\(\pm 1.42\)} & \shortstack[r]{\(3.89\)\\\(\pm 1.32\)} & \shortstack[r]{\(13.80\)\\\(\pm 1.43\)} & \shortstack[r]{\(5.85\)\\\(\pm 0.56\)} \\
Fresh FM & \shortstack[r]{\(8.80\)\\\(\pm 2.81\)} & \shortstack[r]{\(4.01\)\\\(\pm 1.50\)} & \shortstack[r]{\(11.71\)\\\(\pm 1.17\)} & \shortstack[r]{\(5.70\)\\\(\pm 0.49\)} \\
Retain & \shortstack[r]{\(22.98\)\\\(\pm 16.24\)} & \shortstack[r]{\(7.96\)\\\(\pm 1.90\)} & \shortstack[r]{\(25.15\)\\\(\pm 6.63\)} & \shortstack[r]{\(7.86\)\\\(\pm 0.52\)} \\
\bottomrule
\end{tabular*}
\end{minipage}
\end{table*}

\begin{figure*}[!tp]
\centering
\includegraphics[width=1.0\textwidth,trim=0 43bp 0 11bp,clip]{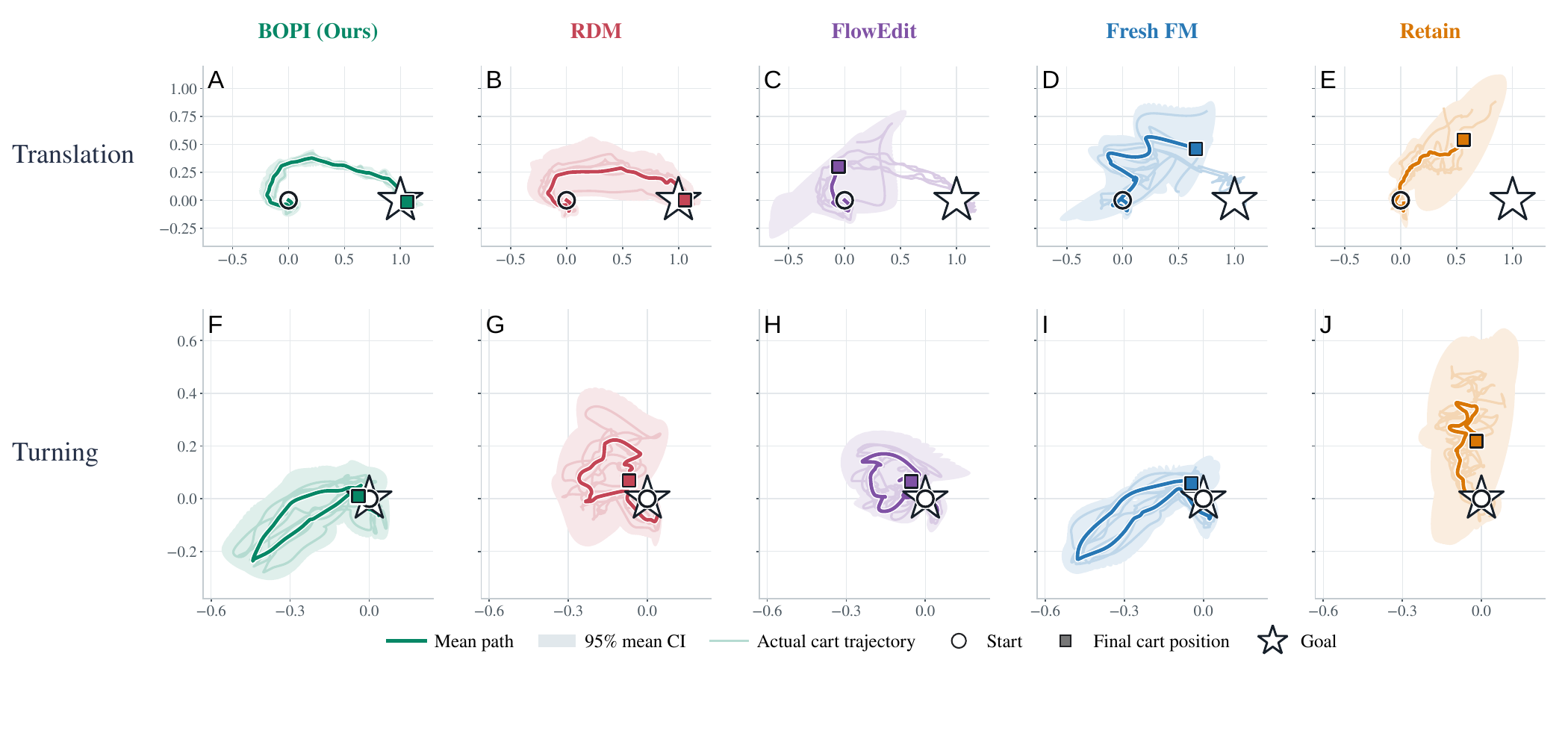}
\caption{\textbf{BOPI produces consistent trajectories for a non-holonomic cart in real-robot experiments.}
The goals are one-meter lateral motion (A--E) and a \(90^\circ\) clockwise turn returning
to the starting position (F--J). Corresponding BOPI photos appear in Fig.~\ref{fig:teaser} (E--J).
Light curves show five estimated cart trajectories; solid curves and shading show their
time-aligned mean and pointwise \(95\%\) confidence regions.
Squares mark the mean at each trajectory's end.
All paths share display-only smoothing (\(\sigma=0.20\,\mathrm{s}\)).
Tables~\ref{tab:real-robot-results} and~\ref{tab:real-robot-rotation} compare complete trials.}
\label{fig:real-robot-trajectories}
\end{figure*}

\subsubsection{Adaptive-repair ablation}\label{adaptive-repair-ablation}

\textcolor{black}{We further test a key design choice: whether to first invert the current plan
(\emph{flow inversion}) or directly add noise, as in RDM (\emph{forward noising}).}
Both use belief overlap to set the depth and reconstruct under the updated belief.
Across the same 18 goals with 64 seeds, forward noising increases
mean final position and yaw errors by \(18.2\%\) and \(20.9\%\), respectively,
compared with flow inversion.
\textcolor{black}{This suggests that BOPI benefits not only from adapting the revision depth,
but also from reversing the old plan through the learned flow rather than perturbing it with fresh noise.}

\subsection{Real-robot experiments}\label{real-robot-experiments}

Figure~\ref{fig:teaser} (A--D) shows adaptable object-centric control across wheeled objects.
We test two goals on a non-holonomic shopping cart: move one meter left
(\mbox{Fig.~\ref{fig:teaser}E--G}), or turn \(90^\circ\) clockwise and return to the starting position
(\mbox{Fig.~\ref{fig:teaser}H--J}), with five trials per method and goal.
All methods start from the same plan, generated under the incorrect assumption
that the rear casters swivel freely.
For hardware safety, we reduce the replanning rate to 1~Hz.

Tables~\ref{tab:real-robot-results} and~\ref{tab:real-robot-rotation} report mean
\(\pm\) sample standard deviation. Final errors measure distance and heading difference to the goal;
tracking RMS measures deviation from the commanded cart trajectory. Metrics use odometry estimates.

Retaining the incorrect initial plan leads to large final position errors
(\(54.23\,\mathrm{cm}\) on average). Fresh FM uses updated estimates but repeatedly replaces
the plan, exhibiting more variable paths (Fig.~\ref{fig:real-robot-trajectories}D,I) and
excessive leaning and sideways leg extension (Fig.~\ref{fig:fresh-fm-failures}).
BOPI instead scales revision to belief agreement, achieving lower mean final position error
and both tracking RMS errors than RDM and FlowEdit on both tasks
(mean final position error:
\(5.78\,\mathrm{cm}\) versus \(8.59\,\mathrm{cm}\) for RDM and \(19.47\,\mathrm{cm}\) for FlowEdit).
Together, these results highlight the value of incorporating new evidence while limiting
disruption to ongoing motion when manipulating objects with initially unknown physical properties.

\begin{figure}[t]
\centering
\newsavebox{\badplanimage}

\sbox{\badplanimage}{\includegraphics[width=\dimexpr(\columnwidth-2pt)/2\relax]
    {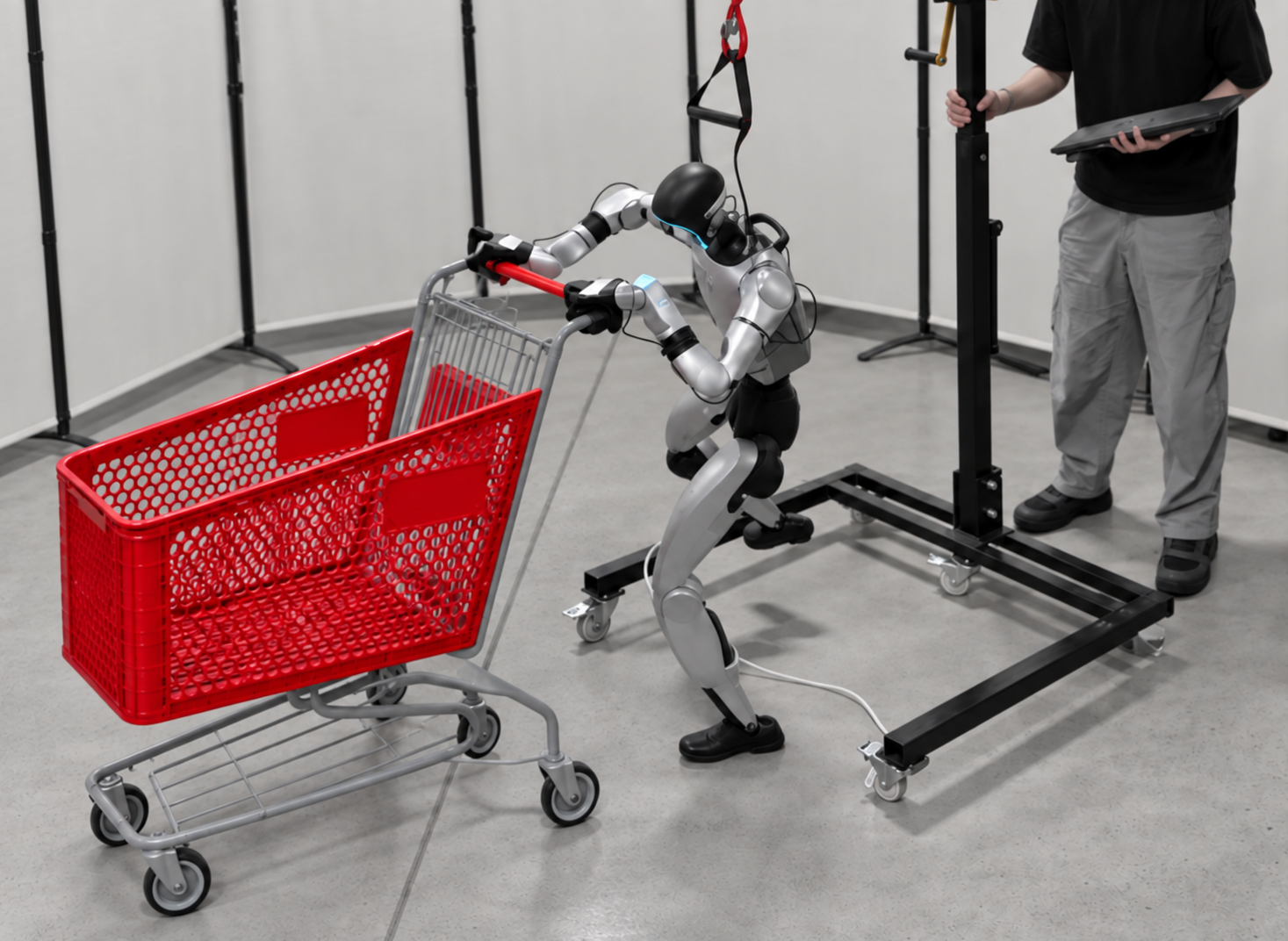}}\usebox{\badplanimage}\llap{\raisebox{3pt}[0pt][0pt]{\normalfont\fontfamily{phv}\mdseries\fontsize{10}{10}\selectfont
    \textcolor{black}{A}\hspace{4pt}}}\hspace{2pt}\sbox{\badplanimage}{\includegraphics[width=\dimexpr(\columnwidth-2pt)/2\relax]
    {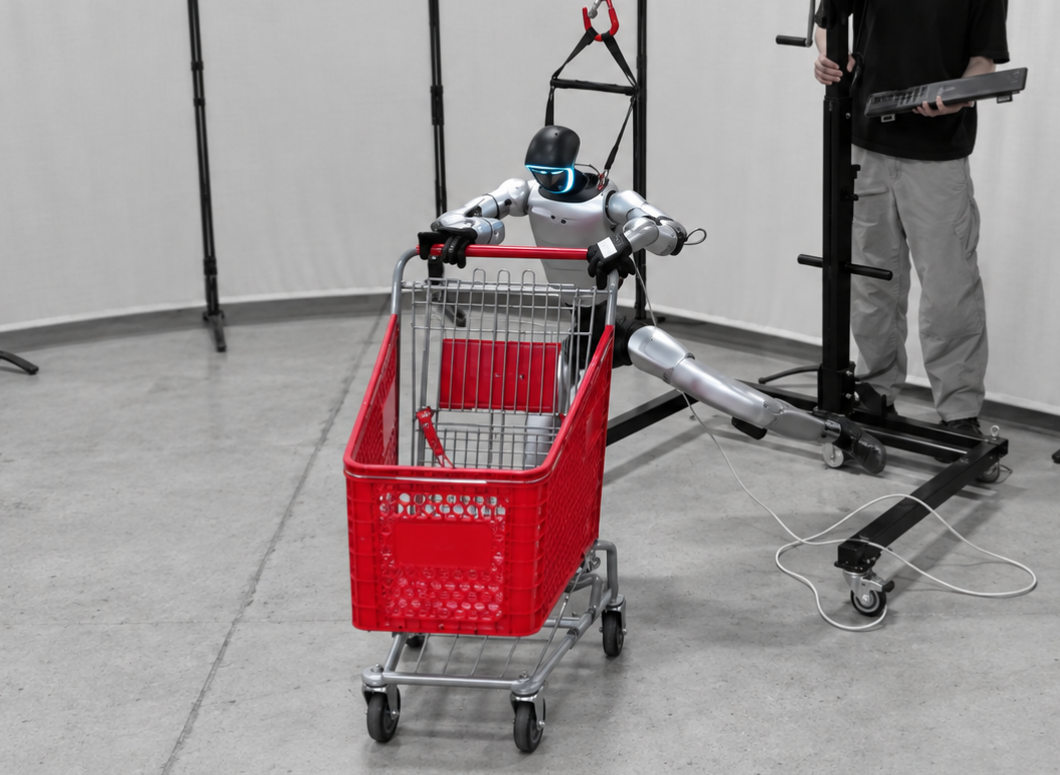}}\usebox{\badplanimage}\llap{\raisebox{3pt}[0pt][0pt]{\normalfont\fontfamily{phv}\mdseries\fontsize{10}{10}\selectfont
    \textcolor{black}{B}\hspace{4pt}}}
\caption{\textbf{Unstable executions of the Fresh FM baseline.}
Excessive leaning (A) and sideways leg extension (B) during lateral pushing illustrate
how replanning can disrupt whole-body motion.
}
\label{fig:fresh-fm-failures}
\end{figure}
\FloatBarrier
\section{Conclusion}\label{conclusion}
We presented an object-centric framework for humanoid cart pushing that combines online property estimation, flow-based planning, and whole-body control. BOPI, its replanning component, uses belief overlap to determine how far to reverse plan generation before reconstructing the timed waypoints under updated estimates. This lets the planner incorporate new evidence while the controller follows the object’s path and adjusts the robot’s body configuration.
Experiments support both the repair mechanism and its use in closed-loop loco-manipulation. BOPI achieves higher success with smaller plan changes in the controlled benchmark and reduces mean final position and yaw errors by approximately 33\% and 47\% relative to fixed-depth RDM across 18 simulated goals.
On a real non-holonomic cart,
BOPI performs translation and turning with lower mean final position and tracking errors than the compared baselines.
Together, these results support our central insight: effective replanning should incorporate new evidence while preserving useful parts of the current plan.
Extending the framework beyond planar cart pushing to three-dimensional goals and non-wheeled objects is a next step toward humanoids that learn to manipulate any unfamiliar objects through interaction.

\vspace{\baselineskip}
\renewcommand{\IEEEbibitemsep}{0.7pt plus 0.25pt}

\end{document}